\documentclass{llncs}
\usepackage{verbatim}

\usepackage{multirow}
\usepackage[bookmarks=true]{hyperref}
\usepackage{graphicx}
\usepackage{subcaption}
\usepackage{amsmath}
\usepackage{color}
\usepackage[ruled, vlined,linesnumbered]{algorithm2e}
\SetAlFnt{\small}

\newcommand{\vv}{\ensuremath{\mathbf v}}
\newcommand{\av}{\ensuremath{\mathbf a}}
\newcommand{\pv}{\ensuremath{\mathbf p}}

\begin{document}

\title{Efficient Multi-Agent Global Navigation Using Interpolating Bridges}
\titlerunning{Efficient Multi-Agent Global Navigation Using Interpolating Bridges}

\author{Liang He\inst{1} and Jia Pan\inst{2} and Dinesh Manocha\inst{1}}
\authorrunning{Liang He et al.} 
%
\tocauthor{Liang He and Jia Pan and Dinesh Manocha}
\institute{UNC Chapel Hill \and City University of Hong Kong}

\maketitle

\begin{abstract}
We present a novel approach for collision-free global navigation for continuous-time multi-agent systems with general linear dynamics. Our approach is general and can be used to perform collision-free navigation in 2D and 3D workspaces with narrow passages and crowded regions.
As part of pre-computation, we compute multiple bridges in the narrow or tight regions in the workspace using kinodynamic RRT algorithms. Our bridge has certain geometric characteristics, that enable us to calculate a collision-free trajectory for each agent using simple interpolation at runtime. Moreover, we combine interpolated bridge trajectories with local multi-agent navigation algorithms to compute global collision-free paths for each agent.
The overall approach combines the performance benefits of coupled multi-agent algorithms with the precomputed trajectories of the bridges to handle challenging scenarios. In practice, our approach can handle tens to hundreds of agents in real-time on a single CPU core in 2D and 3D workspaces.
\end{abstract}

\section{Introduction}
Multi-agent navigation algorithms are widely used for motion planning among static and dynamic obstacles. The underlying applications include cooperative surveillance, sensor networks, swarm navigation, and simulation of animated characters or human crowds in games and virtual worlds. One key problem in multi-agent navigation is the computation of collision-free trajectories for agents, given their own initial and goal positions.

This problem has been studied extensively in robotics, AI, and computer animation. At a broad level, prior approaches can be classified into  \emph{coupled} or \emph{decoupled} planners.
A coupled planner aggregates all the individual robots into one large composite system, and leverages classical motion planners (e.g., sampling-based planners) to compute collision-free trajectories for all agents. On the other hand, decoupled planners computes a trajectory for each robot individually for a short horizon (e.g., a few time-steps), and then performs a velocity coordination to resolve the collision between the local trajectories of all agents. Different techniques have been proposed to compute local collision-free paths or schedule their motion.

Coupled planners are (probabilistically) complete in theory and thus can provide rigorous guarantees about collision avoidance between the agents and the obstacles. However, as the number of agents in the scene increases, the resulting dimension of the system's configuration space increases linearly. As a result, current coupled algorithms are only practical for a few agents. Decoupled planners are generally faster because fewer degrees of freedom are taken into account at a time. Unfortunately, they are usually not complete because, the velocity coordination may not resolve all collisions, and the agents may get stuck in crowded scenarios or block each other (e.g., due to the inevitable collision states~\cite{Fraichard:2003:ICS,Fraichard:2007:SPM}). Thus, in challenging scenarios with narrow passages, a decoupled planner may take a long time or even be unable to find a solution when one exists. In addition, the velocity coordination can be slow in crowded environments, where all agents have to move toward their goals at very small steps.

{\noindent \bf Main Results:} In this paper, we present a novel method to accelerate the performance of decoupled multi-agent planners in crowded or challenging environments. Our approach is limited to scenarios with static obstacles or dynamic obstacles whose trajectories are kn
own beforehand. The key idea is to compute bridges in the narrow passages or challenging areas of the workspace, which are collision-free regions of the workspace and have certain geometric navigation characteristics. As an agent approaches the bridge, we use the precomputed local trajectories associated with that bridge to guide the agent toward the goal position.
We also present an efficient scheduling scheme which enables multiple agents to share a single bridge efficiently.
The overall trajectory of each agent is calculated by using an optimal trajectory generation algorithm~\cite{webb2013kinodynamic} along the interpolated path in the bridge, which combines the efficiency of the decoupled methods with the completeness of the interpolating bridges.
\nocite{Berenson:ICRA:RPP:2012}

A novel component of our approach is the computation of interpolating bridges in 2D and 3D workspaces. Each bridge lies in the collision-free space, and its boundaries are calculated using kinodynamic RRT algorithms. Our approach guarantees that when an agent enters a bridge with a velocity satisfying suitable conditions (as later will be discussed in detail), it can always compute a collision-free trajectory that lies within the bridge. Furthermore, we present an inter-trajectory scheduling scheme for multiple agents sharing a bridge for navigation that has a small runtime overhead. We present efficient algorithms to compute these bridges in 2D and 3D workspaces, to generate trajectories within bridges at runtime, and to schedule agents sharing the bridges for efficient global navigation. We highlight the performance of our method in several challenging 2D and 3D workspsaces with narrow passages. We demonstrate its performance by comparing it with prior multi-agent navigation approaches based on local collision avoidance. In practice, our method can handle about 50 agents in real-time on a single CPU core for both 2D and 3D scenarios.

The rest of the paper is organized as follows. We give a brief survey of prior work in Section~\ref{sec:related}. We introduce our notation and provide an overview of our approach in Section~\ref{sec:overview}. We describe the bridge computation and global multi-agent navigation algorithm in Section~\ref{sec:bridgebuild} and Section~\ref{sec:navigation}, respectively. We analyze our method's properties in Section~\ref{sec:analysis}, and finally  highlight its performance in Section~\ref{sec:experiment}.

\section{Related Work}
\label{sec:related}

There has been extensive work on multi-agent motion planning. This work includes many reactive methods such as RVO~\cite{Jur:2011:RVO}, HRVO~\cite{Snape:2014:SCN}, and their variants. These techniques compute a feasible movement for each agent such that it can avoid other agents and obstacles in a short time horizon. However, they cannot provide mathematical guarantees about whether or not agents can always find collision-free trajectories. In particular, they may not be able to avoid the inevitable collision states (ICS)~\cite{Fraichard:2003:ICS,Fraichard:2007:SPM,John:1985:MPP,Bekris:2012:SDM} in the configuration space, due to robots' dynamical constraints or obstacles in the scenario. Some methods~\cite{Bekris:2012:SDM,he2013meso,he2016proxemic} provide partial solutions to these problems, but they still cannot guarantee avoidance of all ICS in the long horizon while working in a crowd scenario with narrow passages. Actually, even for a scenario without static obstacles, it is still difficult to achieve robust collision-avoidance coordination when there is a large number of agents~\cite{solovey2015motion,solovey2014hardness}.

The simplest solution to the difficulty of inevitable collision states is to design suitable protocols for multi-robot coordination/interaction~\cite{Dresner:2008:MAA,Alami:1998:MRC}. Some other approaches precompute roadmaps or corridors in the entire workspace to achieve high-quality path planning~\cite{geraerts2008using,Wein:2008:PHQ}. However, these methods are not complete and may provide sub-optimal trajectories. Our method also leverages pre-computed bridges to deal with the navigation challenges in narrow or crowded regions. However, the bridges used in our approach have special properties beneficial for efficient global navigation in challenging areas in the workspace.

Centralized multi-agent navigation approaches usually leverage global single-robot planning algorithms (such as PRM or RRT) to compute a roadmap or grids for the high-level coordination~\cite{Berg:2005:RMP,Berg:2005:PMP}. Compared to the decentralized methods, these algorithms compute all agents' trajectories
simultaneously and thus can better handle the complex interactions among agents. These methods can also be extended to handle non-holonomic multi-agent systems (e.g., systems composed of differential-drive robots), by using local planners like RRT~\cite{Lavalle:2001:RKP}, RRT$^*$~\cite{Karaman:2010:OKM}, or other algorithms that can deal with differential dynamics~\cite{webb2013kinodynamic}. In addition to their benefit in terms of finding feasible trajectories, the centralized algorithms can avoid deadlock cases by leveraging high-level scheduling or coordination strategies, either coupled~\cite{Cap:2013:MRS,Ferguson:2006:RRR} or decoupled~\cite{Katsev:2013:EFP,Sanchez:2002:UPP,Jur:2009:RSS,Luna:2011:PSF}. In general, these strategies only work in theory, because they have to ignore the robot's dynamics and assume the robots to be operating in a discrete state space. Finally, centralized multi-agent navigation algorithms are computationally expensive and can sometimes be too slow for real-world applications.

\section{Overview}
\label{sec:overview}

\subsection{Problem Definition}
Our goal is to enable a group of agents to reach their individual goals in a safe and efficient manner. Our approach is designed for continuous-time multi-agent systems with general linear dynamics. During the navigation, the agents should avoid collisions with static obstacles in the environment as well as with other agents.

This problem can be formally defined as follows. We take as given a set of $n$ decision-making agents sharing an environment consisting of obstacles. For simplicity, we assume the geometric shape of each agent is represented as a disc of radius $r$, and its current position and velocity are denoted as $\pv$ and $\vv$, respectively. We also assume kinodynamic constraints on the agent's motion: $\|\vv\| \leq v_{\max}$ and $\|\av\| \leq a_{\max}$, where $v_{\max}$ and $a_{\max}$ are the maximum allowed velocity and acceleration for the agent's velocity $\vv$ and acceleration $\av$. The task of each agent is to compute a trajectory toward its goal position $G$ from its initial position $I$. The trajectory should be collision-free and also satisfy other dynamics constraints.

\begin{figure}
    \centering
    \includegraphics[width=0.8\linewidth]{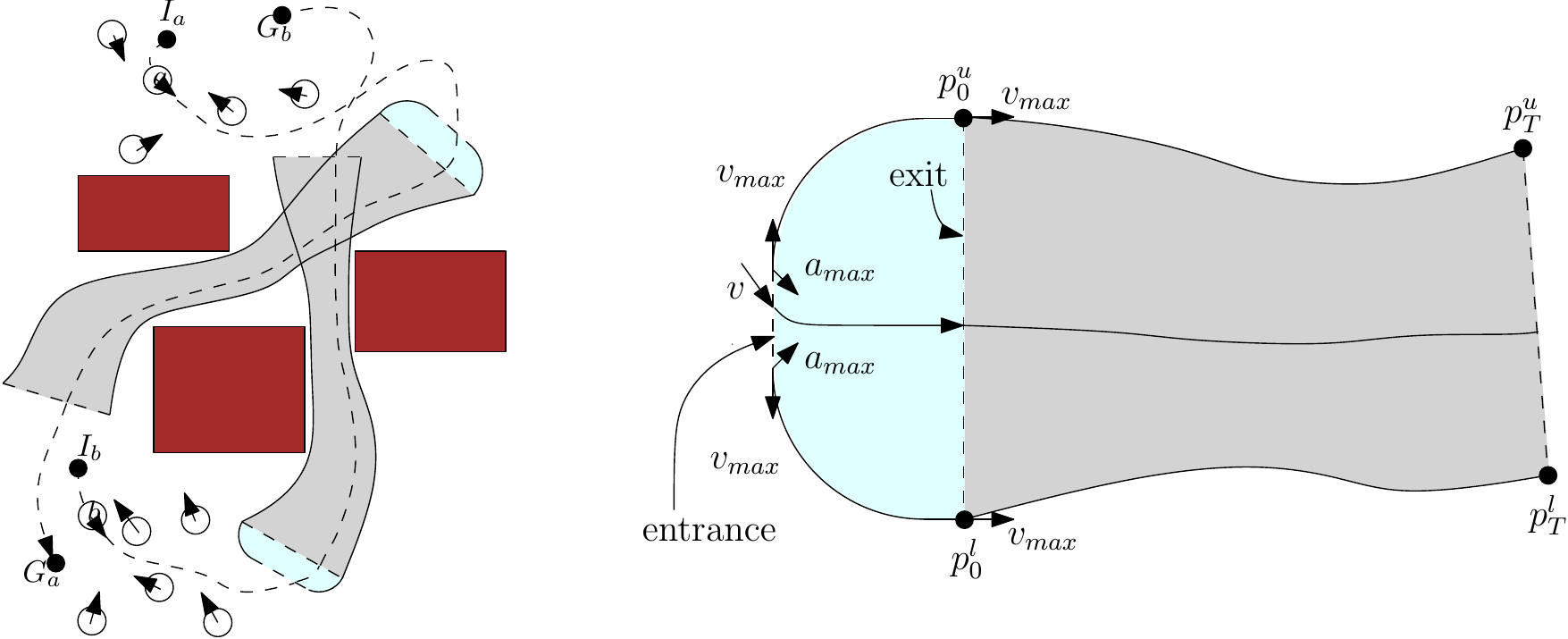}
    \caption{An overview of our approach: The left figure shows a scene with three obstacles (the brown boxes), agents (the circles) want to reach their individual goals $G_a, G_b, ...$ from their initial positions $I_a, I_b, ...$. We use two bridges (the gray regions) in the environment to help the agents navigate through the crowded areas efficiently. The agent $a$ first goes toward the entrance (the cyan region) of one bridge following a path from $I_a$ to the entrance. Next, $a$ goes inside the entrance where its velocity is gradually adjusted before entering the bridge. Finally, $a$ goes through the bridge and arrives at its goal $G_a$. The agent $b$ arrives at its goal $G_b$ through another bridge. The dashed line denotes an agent's trajectory. These bridges enable global navigation with safety guarantees. The right figure shows the entrance to a bridge: Given an agent with velocity $\vv$, we use the entrance to gradually change the agent's velocity so that while entering the bridge the agent's velocity is equal to the initial velocity of the bridge's boundary trajectories. }
    \vspace*{-0.2in}
    \label{fig:overview}
\end{figure}

\subsection{Our Approach}
Given a crowded scenario as shown in Figure~\ref{fig:overview}, the most efficient navigation strategy for a group of agents should be as follows: in the open space outside the narrow passage, each agent can navigate according to a path computed by optimal motion planning approaches; in the narrow area, agents should maintain a line or some formation, and then pass through the narrow region in some order.

For this purpose, we allocate a set of guidance channels called \emph{bridges} in the narrow parts of the workspace, as shown in Figure~\ref{fig:overview}. Each bridge lies totally in the collision-free subset of the configuration space and its boundary is computed using RRT-based collision-free paths. Furthermore, each bridge has an attractive characteristic such that for an agent entering the bridge with its velocity satisfying a special constraint, we can always compute a trajectory that completely lies inside the bridge. In order to pass through the narrow passages of the workspace, an agent will first move toward one of the bridges according to its own dynamics. However, the agent may arrive at the bridge with an arbitrary velocity. Thus, the agent must first enter an area called \emph{entrance} in front of the bridge. The entrance will gradually adjust the agent's velocity to make sure it satisfies the requirements with respect to that bridge. Once the agent leaves the entrance and enters inside the bridge, it can follow a collision-free trajectory to pass through the crowded or narrow region efficiently and safely. After leaving the bridge, the agent can switch back to the local coordination strategy and move toward its goal. In this way, the bridges can be viewed as a "highway system" for the agents to efficiently travel through challenging scenarios.

While we can build one bridge for each agent, this may result in many overlaps between the bridges, and therefore the agents will need to slow down or even wait while moving along the bridges. A better solution is to build a few bridges that can be reached by all agents, and allow multiple agents to share one bridge by using suitable scheduling schemes.


\section{Bridge Construction}
\label{sec:bridgebuild}
In this section, we describe the details of how to automatically compute the bridges in the workspace. We will first describe our approach for a 2D workspace, which will later be extended to 3D workspaces.

We start from a zero-width bridge, which is a collision-free trajectory connecting a pair of start and goal positions in the crowded region. The trajectory is computed using a kinodynamic RRT planner~\cite{Lavalle:2001:RKP}, and has an initial velocity of $\vv_0$ with a magnitude $v_{\max}$. We then incrementally enlarge the bridge's width until it touches one obstacle in the scenario. The built bridge has a beneficial property alllowing that as long as an agent enters the bridge with the velocity $\vv_0$, it can always use an efficient interpolation scheme to calculate a collision-free trajectory passing through the bridge.
However, an agent may enter a bridge with an arbitrary velocity and thus violate the bridge's constraint on the entering velocity. To solve this problem, we add an entrance region in front of the bridge,
which provides sufficient room for an agent to gradually adjust its velocity toward $\vv_0$ in order to leverage the bridge for safe and efficient navigation. Examples of 2D bridges and their entrances are shown in Figures~\ref{fig:overview}, ~\ref{fig:2d_bridge_interpolation}.
Finally, for each entrance we compute a backward-reachable set, i.e., the set of positions from which an agent can reach the entrance. We also compute a forward-reachable set for the end line of the bridge, i.e., the set of positions that can be reached by an agent coming out of the bridge. All agents with their start positions inside the backward-reachable set and goal positions inside the forward-reachable set will leverage the constructed bridge for navigation. We repeatedly add more bridges for the remaining agents until each agent has one associated bridge.

\vspace*{-0.1in}

\subsection{Iterative Bridge Enlargement}
We first use the kinodynamic RRT planner~\cite{Lavalle:2001:RKP} to connect a pair of starting and goal positions in the crowded scenarios, and the result is a trajectory $\{\pv_0, ..., \pv_T\}$. We consider the trajectory as a bridge with zero width, and its two boundary trajectories are $\{\pv_0^u, ..., \pv_T^u\}$ and $\{\pv_0^l, ..., \pv_T^l\}$ respectively, where $\pv_i = \pv_i^u = \pv_i^l$ are overlapped points. Starting from this initial bridge, we incrementally enlarge the bridge's width until it hits one obstacle. The algorithm is as shown in Algorithm~\ref{algo:bridge_enlarge}.

\begin{algorithm}[!ht]
  \SetKwInOut{Input}{input}\SetKwInOut{Output}{output}
  \Input{Start and goal points $\pv_0$ and $\pv_T$}
  \Output{A valid bridge}
  \BlankLine
  \tcc{Initialize with a zero-width bridge}
  $\{\pv_0, ..., \pv_T\} \leftarrow$ $RRT(\pv_0, \pv_T)$ \\
  \tcc{Set the overlapped bridge upper and lower boundaries}
  $\{\pv_0^u, ..., \pv_T^u\} \leftarrow \{\pv_0, ..., \pv_T\}$ \\
  $\{\pv_0^l, ..., \pv_T^l\} \leftarrow \{\pv_0, ..., \pv_T\}$ \\

  \While{true}
  {
     \tcc{Check collision for the bridge area between two boundaries}
     collision $\leftarrow$ bridgeCollide$(\{\pv_0^u, ..., \pv_T^u\}, \{\pv_0^l, ..., \pv_T^l\})$ \\
     \If{collision}{break}
     \tcc{Change the start and goal positions of two boundaries}
     $\pv_0^u \leftarrow \pv_0^u + \Delta \pv_0$, $\pv_T^u \leftarrow \pv_T^u + \Delta \pv_T$ \\
     $\pv_0^l \leftarrow \pv_0^l - \Delta \pv_0$, $\pv_T^l \leftarrow \pv_T^l - \Delta \pv_T$ \\
     \tcc{Generate new bridge boundaries using RRT}
     $\{\pv_0^u, ..., \pv_T^u\} \leftarrow $RRT$(\pv_0^u, \pv_T^u)$ \\
     $\{\pv_0^l, ..., \pv_T^l\} \leftarrow $RRT$(\pv_0^l, \pv_T^l)$ \\
  }
  \Return A bridge $b$ with boundaries $\{\pv_0^u, ..., \pv_T^u\}$ and $\{\pv_0^l, ..., \pv_T^l\}$
  \caption{Construct a 2D bridge}
  \label{algo:bridge_enlarge}
\end{algorithm}

\vspace*{-0.15in}

\subsection{Trajectory Generation in a Bridge}
\label{sec:bridgebuild:bridgeinterp}
Given a 2D bridge, as shown in Figure~\ref{fig:2d_bridge_interpolation}, we can show that if the agent enters the bridge with a velocity equal to the initial velocities of the bridge's two boundary trajectories, we can use an efficient interpolation scheme to generate a trajectory for the agent such that the agent will always stay inside the bridge and compute a path that is collision-free with all static obstacles.

\begin{figure}
    \centering
    \includegraphics[width=0.6\linewidth]{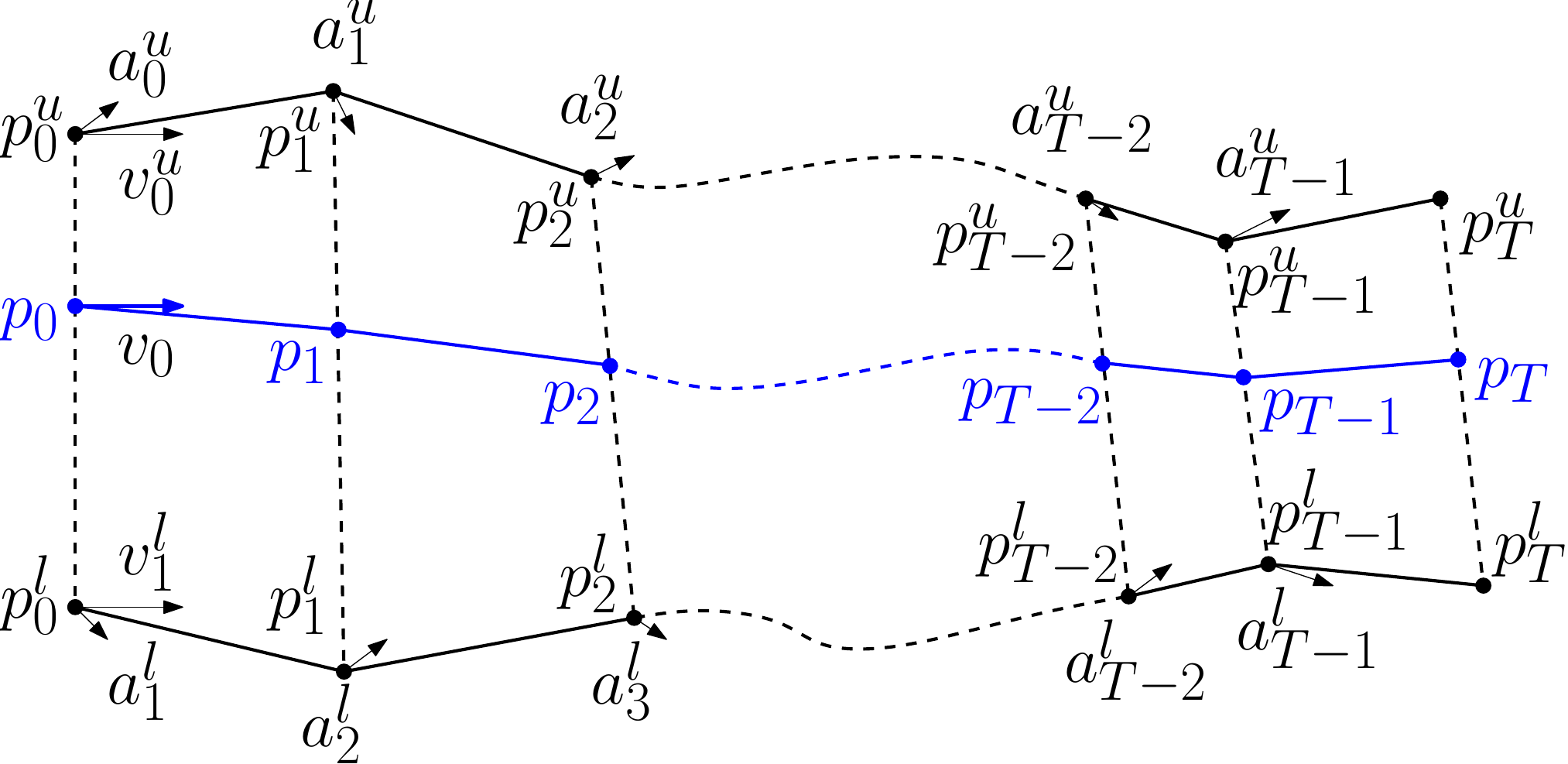}
    \caption{Trajectory interpolation in a 2D bridge: The bridge is bounded by an upper trajectory $\{\pv_0^u,\pv_1^u, ..., \pv_{T-1}^u, \pv_T^u\}$ and a lower trajectory $\{\pv_0^l, \pv_1^l, ..., \pv_{T-1}^l, \pv_T^l\}$, where the velocity and acceleration at each trajectory point $\pv_i$ are $\vv_i$ and $\av_i$ respectively. The start line and end line of the bridge are $\pv_0^u \pv_0^l$ and $\pv_T^u \pv_T^l$ respectively. The agent enters the bridge at the position $\pv_0$ on the start line with a velocity $\vv_0$. If $\vv_0 = \vv_0^u = \vv_0^l$, the trajectory interpolation scheme can compute the agent's trajectory as $\{\pv_0, ..., \pv_T\}$ which will be located completely inside the bridge. }
    \vspace*{-0.15in}
    \label{fig:2d_bridge_interpolation}
\end{figure}

The interpolation details are described in Algorithm~\ref{algo:2d_bridge_interpolation}. In particular, we choose the acceleration of the agent at each time step $i \Delta t$ as a linear interpolation of the accelerations of the corresponding waypoints on the boundary trajectories: $\av_i = (1 - r) \av_i^u + r \av_i^l$. The interpolation coefficient $r$ is the ratio based on which the agent's initial position $\pv_0$ partitions the bridge area's start line $\pv_0^u\pv_0^l$: $r = \frac{|\pv_0^u \pv_0|}{|\pv_0 \pv_0^l|}$. We can prove that in this way, the generated trajectory $\{\pv_0, ..., \pv_T\}$ always stays inside the bridge:
\begin{theorem}
\label{thm:2d_bridge_interpolation}
{\em The interpolated trajectory $\{\pv_0, ..., \pv_T\}$ always stays inside the bridge}.
\end{theorem}
\begin{proof} We first show that the interpolated trajectory has the following two properties $\pv_i = (1-r) \pv_i^u + r \pv_i^l$ and $\vv_i = (1-r) \vv_i^u + r \vv_i^l$, for $i = 1, ..., T$. We can prove these two statements by induction on $i$. If $i=1$, $\pv_0 = (1-r) \pv_0^u + r \pv_0^l$ is trivial because this is the definition of $r$; and since $\vv_0 = \vv_0^u = \vv_0^l$ as required by the bridge's definition, $\vv_0 = (1-r) \vv_0^u + r \vv_0^l$ is also obvious. Now consider the case in which $i > 1$:
\begin{align}
\vv_i & =  \vv_{i-1} + \av_{i-1} \Delta t
=  (1-r) \vv_{i-1}^u + r \vv_{i-1}^l + [ (1-r) \av_i^u + r \av_i^l] \Delta t \\
& =  (1-r) [\vv_{i-1}^u + \av_i^u \Delta t] + r [\vv_{i-1}^l + \av_i^l \Delta t]
=  (1-r) \vv_i^u + r \vv_i^l.  \notag
\end{align}
Similarly, we have
\begin{align}
\pv_i & = \pv_{i-1} + \vv_{i-1} \Delta t + \frac{1}{2} \av_{i-1} (\Delta t)^2
 = (1-r) \pv_i^u + r \pv_i^l.
\end{align}
Based on the induction hypothesis, the waypoints of the interpolated trajectory will always be a linear interpolation of the corresponding waypoints of the bridge's two boundary trajectories. Thus, the interpolated trajectory must be contained inside the bridge.
\end{proof}

\begin{algorithm}[!ht]
  \SetKwInOut{Input}{input}\SetKwInOut{Output}{output}
  \Input{The bridge's two boundary trajectories $\{\pv_0^u, ..., \pv_T^u\}$ and $\{\pv_0^l, ..., \pv_T^l\}$, along with each waypoint $\pv_i$'s velocity $\vv_i$ and acceleration $\av_i$. The initial position $\pv_O$ and velocity $\vv_0$ when the agent enters the bridge. }
  \Output{The agent's trajectory $\{\pv_0, ..., \pv_T\}$}
  \BlankLine
  \tcc{Compute the ratio into which $\pv_0$ partitions the start line $\pv_0^u\pv_0^l$}
  $r = \frac{|\pv_0^u \pv_0|}{|\pv_0 \pv_0^l|}$ \;
  \For {$i = 1, ..., T-1$}
  {
    $\av_i = (1 - r) \av_i^u + r \av_i^l$ \;
    $\pv_{i+1} = \pv_i + \vv_i \Delta t + \frac{1}{2} \av_i (\Delta t)^2$ \;
    $\vv_{i+1} = \vv_i + \av_i \Delta t$ \;
  }
  \caption{Trajectory generation in a 2D bridge}\label{algo:2d_bridge_interpolation}
\end{algorithm}

\vspace*{-0.1in}

\subsection{2D Entrance Construction and Trajectory Generation}
When an agent enters the bridge, its velocity must be the same as the initial velocity of the bridge's boundary trajectories. However, agents may arrive at the bridge with an arbitrary velocity, and thus the bridge may not be able to generate a trajectory for an agent that is completely inside the bridge. Our solution is to leave some space in front of the bridge called the \emph{entrance} to the bridge. This space works like a buffer zone where the agent can gradually adjust its velocity to meet the bridge's requirement.

An example of the 2D entrance is shown in the right side in Figure~\ref{fig:overview}, which includes four parts: one start line, one end line, and two boundary curves. The end line of the entrance is also the start line of the bridge. Each boundary curve is composed of a parabola followed by a line. The two parabolas correspond to the trajectories in which the agent's velocity is gradually rotated by $90$ degrees from the initial
$\pm v_{\max} \hat{\mathbf y}$ to the final $v_{\max} \hat{\mathbf x}$. The velocity change is achieved by using a constant acceleration $\mathbf a$, the magnitude of which is $a_{\max}$ and the direction of which is $\frac{\sqrt{2}}{2} \hat{\mathbf x} \mp \frac{\sqrt{2}}{2} \hat{\mathbf x}$. For convenience, we assume the bridge's start line is along the $\hat{\mathbf y}$ direction. The two lines following the parabolas have a length $(\sqrt{\frac{2}{3}} - \sqrt{\frac{1}{2}}) \frac{v_{\max}^2}{a_{\max}}$ each.

The shape of the entrance chosen has the property that, for any agent with an arbitrary velocity $\vv$ ($|\vv| \leq v_{\max}$) entering the entrance, we can always find a sequence of accelerations to gradually change its velocity from $\vv$ to $v_{\max} \hat{\mathbf x}$(i.e., with a magnitude of $v_{\max}$ and perpendicular to the bridge's start line). After this velocity adjustment, the agent will have a velocity equal to the initial velocity of the bridge's boundary trajectories when entering the bridge.
According to Section~\ref{sec:bridgebuild:bridgeinterp}, this property is required for the trajectory interpolation inside the bridge.

We can choose the acceleration such that, during the adjustment of the velocity, the agent will always stay inside the area of the entrance and is therefore, guaranteed to avoid collisions with other agents and obstacles. In algorithm ~\ref{algo:2d_bridge_interpolation}, we provide the details about velocity adjustment. In particular, given an agent entering the entrance at a velocity $\vv = (v_x, v_y)$, our goal is to increase $v_x$ to $v_{\max}$ and decrease $v_y$ to $0$. We first determine a suitable acceleration $(a_x, a_y)$ by comparing the speed gaps in both directions (line 1). This acceleration guarantees that $v_x$ will increase to $v_{\max}$ after $v_y$ arrives at $0$, which is necessary to keep the agent inside the entrance region. After $v_y$ becomes $0$, the acceleration along the $\hat{\mathbf y}$ will become zero, while the acceleration along $\hat{\mathbf x}$ remains the same. The agent moves on until its velocity is $v_{\max} \hat{\mathbf x}$. After that, if the agent has not reached the exit line, it will continue with the current velocity. After computing the extreme positions at which the resulting trajectory will arrive, we can show that the trajectory will never go out of the entrance region and thus is guaranteed to be collision-free with static obstacles.

\subsection{Bridges' Forward and Backward Reachable Regions}
After generating the bridge and its entrance, we then compute its forward and backward reachable regions.
We denote $\mathcal R^+[\pv, \tau]$ as the forward-reachable set of a position $\pv$, (i.e., the set of positions that can be reached from $\pv$ with time less than $\tau$), and denote $\mathcal R^-[\pv, \tau]$ as the backward-reachable set, (i.e., the set of positions that can reach $\pv$ with time less than $\tau$):
\begin{align}
\mathcal R^+(\pv, \tau) &= \{\pv' | \emph{time}(\pv, \pv') < \tau \} \\
\mathcal R^-(\pv, \tau) &= \{\pv' | \emph{time}(\pv', \pv) < \tau \},
\end{align}
where $\emph{time}(\mathbf x, \mathbf y)$ measures the time an agent needs to move from a starting position $\mathbf x$ to a goal position $\mathbf y$. Both reachability sets can be efficiently estimated using the method proposed in~\cite{webb2013kinodynamic}.

Leveraging the concept of forward and backward reachable sets, we can compute the forward and backward reachable regions for the entire bridge as
\begin{align}
\mathcal R^+(\emph{bridge}, \tau) &= \bigcup_{\pv \in \emph{bridge end line} } \mathcal R^+(\pv, \tau) \\
\mathcal R^-(\emph{bridge}, \tau) &= \bigcup_{\pv \in \emph{entrance start line} } \mathcal R^-(\pv, \tau).
\end{align}

\subsection{Bridge Assignment among Agents}
We have finished describing how to build a bridge, its entrance, and the corresponding reachability regions. We will now discuss how to build a set of bridges for all the agents in the scenario, as well as how to assign a bridge to each agent. Our solution is to first build a bridge for a single agent. Next, we check whether or not there are any other agents whose initial and goal positions are within the backward and forward reachable regions of this bridge. If so, these agents can also leverage this bridge for navigation, and thus we assign this bridge to these agents. We repeatedly add more bridges until all agents have one associated bridge. The algorithm is as shown in Algorithm~\ref{algo:bridge_generator}.

\begin{algorithm}[!ht]
  \SetKwInOut{Input}{input}\SetKwInOut{Output}{output}
  \Input{Initial and goal configurations pairs $IG= \{(I_i, G_i)\}_{i=0}^{n-1}$ for all $n$ agents}
  \Output{A set of bridges $B = \{B_j\}$}
  \BlankLine
  $B \leftarrow \emptyset$ \\
  \tcc{compute the bridges for all of the agents }
  \While {$IG \neq \emptyset$}
  {
    Select two configurations $(I, G)$ from the set $IG$ \\
    Construct a 2D bridge $b$ and its entrance using Algorithm~\ref{algo:bridge_enlarge}  \\
    Compute the reachability regions $\mathcal R^+(b, \tau)$ and $\mathcal R^-(b, \tau)$ \\
    \tcc{compute the bridge for the agents }
    \For{$(I_i, G_i) \in IG$}
    {
       \If{$I_i \in \mathcal R^-(b, \tau)$ \emph{and} $G_i \in \mathcal R^+(b, \tau)$}
       {Remove $(I_i,G_i)$ from $IG$ \\ Assign bridge $b$ to the agent $i$ \\ Add bridge $b$ to $B$}
    }
  }
  \Return $B$
  \caption{Generating bridges for all agents}
  \label{algo:bridge_generator}
\end{algorithm}

\vspace*{-0.1in}

\subsection{3D Bridge and Entrance}
The bridge and entrance algorithms described above can be extended to the 3D workspace.
Figure~\ref{fig:3d_bridge_interpolation} illustrates the 3D bridge and the corresponding trajectory interpolation algorithm. The boundary of the 3D bridge is composed of $K$ trajectories $\{\pv_0^k, \pv_1^k, ..., \pv_{T-1}^k, \pv_T^k\}$, computed by the RRT algorithm, where $k=1,..., K$. The initial points $\pv_0^k$ of all these $K$ trajectories are located on the same plane(called the start plane of the bridge), and their initial velocities $\vv_0^k$ are the same, all having the magnitude of $v_{\max}$. Given an agent entering the bridge at position $\pv_0$ on the start plane, we can compute its trajectory using interpolation as follows. First, from all the initial points, we can select three points whose convex combination can be used for the point $\pv_0$.
W.l.o.g., we assume that these three points are $\pv_0^0, \pv_0^1, \pv_0^2$, and their convex combination is $\pv_0 = u \pv_0^0 + v \pv_0^1 + w\pv_0^2$, where $0\leq u, v, w\leq 1$ and $u+v+w = 1$. Similar to Theorem~\ref{thm:2d_bridge_interpolation}, we can show that if we choose the acceleration at time $t$ to be $\av_t = u \av_t^0 + v \av_t^1 + w\av_t^2$, the resulting trajectory will have velocity $\vv_t = u \vv_t^0 + v \vv_t^1 + w\vv_t^2$ and position $\pv_t =  u \pv_t^0 + v \pv_t^1 + w\pv_t^2$ at time $t$. It follows that this trajectory will always stay inside the 3D bridge.
The 3D entrance construction and trajectory generation algorithms are also similar to the 2D case, except that the entrance's boundary is now a surface bounded by a series of $K$ parabolas.

\begin{figure}
    \centering
    \includegraphics[width=0.8\linewidth]{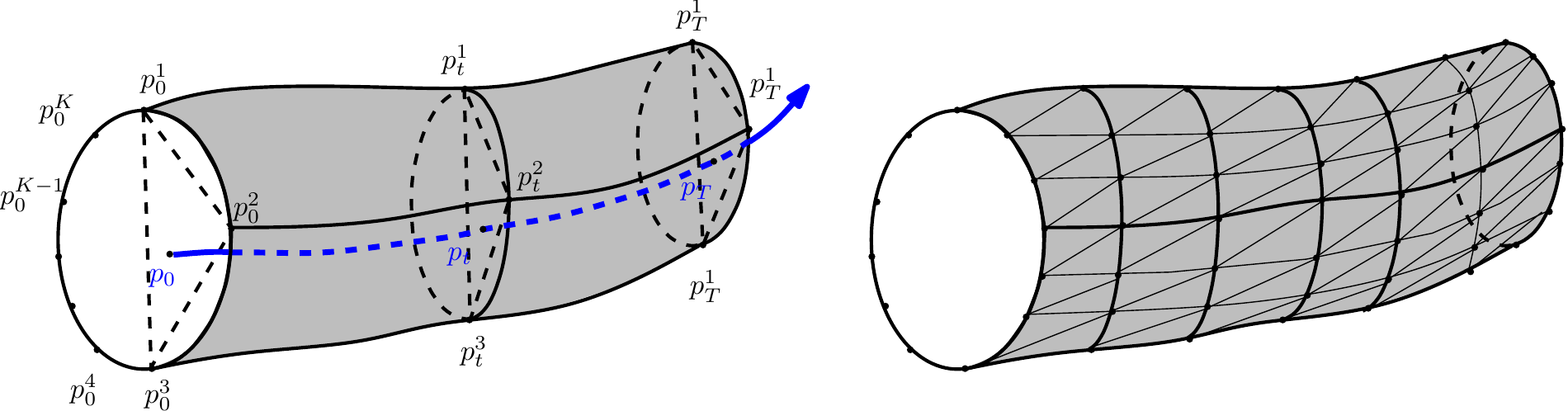}
    \caption{The left figure shows the trajectory interpolation in a 3D bridge: the bridge is bounded by a set of boundary trajectories $\{\pv_0^k, \pv_1^k, ..., \pv_{T-1}^k, \pv_T^k\}$, where $k=1, ..., K$. The agent enters the bridge at the position $\pv_0$ on the start plane. The trajectory interpolation is performed by first choosing three points from $\{\pv_i^k\}$ that can provide a convex combination for the point $\pv_0$. Here we assume the first three points are chosen, and thus $\pv_0 = u \pv_0^1 + v \pv_0^2 + w\pv_0^3$, where $0 \leq u,v,w \leq 1$ and $u+v+w = 1$. Similar to the 2D bridge, we can always choose a suitable acceleration $\av_t$  so that the resulting trajectory $\{\pv_0, ..., \pv_T\}$ completely locates inside the bridge. The right figure shows the triangulation for the 3D bridge shown in the left figure. In this case, each point is one waypoint in one of the $K$ boundary trajectories. }
    \label{fig:3d_bridge_interpolation}
\end{figure}

\vspace*{-0.1in}
\subsection{Collision Checking for a Bridge}
\label{sec:method:collision}
Given a bridge, we need to make sure that it is collision-free. To perform collision checking on a bridge, we triangulate the bridge boundary, and then perform collision checking between the triangulated bridge boundary with static obstacles in the environment. The left side in figure~\ref{fig:3d_bridge_interpolation} shows the triangulation result for the 3D bridge . The collision checking is performed using traditional bounding volume techniques~\cite{Gottschalk:1996:OHS}.

\section{Bridges' Forward and Backward Reachable Regions}
After generating the bridge and its entrance, we then compute its forward and backward reachable regions.
We denote $\mathcal R^+[\pv, \tau]$ as the forward-reachable set of a position $\pv$,( i.e., the set of positions that can be reached from $\pv$ with time less than $\tau$), and denote $\mathcal R^-[\pv, \tau]$ as the backward-reachable set, (i.e., the set of positions that can reach $\pv$ with time less than $\tau$):
\begin{align}
\mathcal R^+(\pv, \tau) &= \{\pv' | \emph{time}(\pv, \pv') < \tau \} \\
\mathcal R^-(\pv, \tau) &= \{\pv' | \emph{time}(\pv', \pv) < \tau \},
\end{align}
where $\emph{time}(\mathbf x, \mathbf y)$ measures the time an agent needs to move from a starting position $\mathbf x$ to a goal position $\mathbf y$. Both reachability sets can be efficiently estimated using the method proposed in~\cite{webb2013kinodynamic}.

Leveraging the concept of forward and backward reachable sets, we can compute the forward and backward reachable regions for the entire bridge as
\begin{align}
\mathcal R^+(\emph{bridge}, \tau) &= \bigcup_{\pv \in \emph{bridge end line} } \mathcal R^+(\pv, \tau) \\
\mathcal R^-(\emph{bridge}, \tau) &= \bigcup_{\pv \in \emph{entrance start line} } \mathcal R^-(\pv, \tau).
\end{align}

\subsection{Bridge Assignment among Agents}
We have finished describing how to build a bridge, its entrance, and the corresponding reachability regions. We will now discuss how to build a set of bridges for all the agents in the scenario, as well as how to assign a bridge to each agent. Our solution is to first build a bridge for a single agent. Next, we check whether or not there are any other agents whose initial and goal positions are within the backward and forward reachable regions of this bridge. If so, these agents can also leverage this bridge for navigation, and thus we assign this bridge to these agents. We repeatedly add more bridges until all agents have one associated bridge. The algorithm is as shown in Algorithm~\ref{algo:bridge_generator}.

\begin{algorithm}[!ht]
  \SetKwInOut{Input}{input}\SetKwInOut{Output}{output}
  \Input{Initial and goal configurations pairs $IG= \{(I_i, G_i)\}_{i=0}^{n-1}$ for all $n$ agents}
  \Output{A set of bridges $B = \{B_j\}$}
  \BlankLine
  $B \leftarrow \emptyset$ \\
  \tcc{compute the bridges for all of the agents }
  \While {$IG \neq \emptyset$}
  {
    Select two configurations $(I, G)$ from the set $IG$ \\
    Construct a 2D bridge $b$ and its entrance using Algorithm~\ref{algo:bridge_enlarge}  \\
    Compute the reachability regions $\mathcal R^+(b, \tau)$ and $\mathcal R^-(b, \tau)$ \\
    \tcc{compute the bridge for the agents }
    \For{$(I_i, G_i) \in IG$}
    {
       \If{$I_i \in \mathcal R^-(b, \tau)$ \emph{and} $G_i \in \mathcal R^+(b, \tau)$}
       {Remove $(I_i,G_i)$ from $IG$ \\ Assign bridge $b$ to the agent $i$ \\ Add bridge $b$ to $B$}
    }
  }
  \Return $B$
  \caption{Generating bridges for all agents}
  \label{algo:bridge_generator}
\end{algorithm}

\subsection{Global Navigation using Bridges}
Once the bridges and their entrances are computed, we can leverage them for efficient multi-agent global navigation. Each given agent, it will first move toward the bridge assigned to it along an optimal trajectory computed using~\cite{webb2013kinodynamic}. Once it reaches the entrance, it can enter the entrance and then go through the crowded or narrow area by following the trajectory interpolated by the bridge. After leaving the bridge, the agent switches back to moving toward its individual goal following an optimal trajectory computed using~\cite{webb2013kinodynamic}.

\vspace*{-0.1in}

\section{Global Navigaion using Bridges}
\label{sec:navigation}

Once the bridges and their entrances are computed, we can leverage them for efficient multi-agent global navigation. Each given agent, it will first move toward the bridge assigned to it along an optimal trajectory computed using~\cite{webb2013kinodynamic}. Once it reaches the entrance, it can enter the entrance and then go through the crowded or narrow area by following the trajectory interpolated by the bridge. After leaving the bridge, the agent switches back to moving toward its individual goal following an optimal trajectory computed using~\cite{webb2013kinodynamic}.

\subsection{Inter-Trajectory Scheduling}
\label{sec:scheduling}
However, there is still one unresolved situation: several agents may try to leverage the bridge at the same time, and this may result in collisions between the agents inside the bridge. To avoid this problem, we use a scheduling scheme among the agents so that they can share the bridge in a safe and efficient manner for collision-free navigation. This is achieved by inter-trajectory scheduling.

Algorithm~\ref{algo:Planner} shows a simple scheduling algorithm among trajectories. In particular, we check whether a planned trajectory will collide with any previously scheduled trajectory. If so, we delay the trajectory for a small time $\delta t$. This process continues until all collisions among trajectories are resolved.

\begin{algorithm}[!ht]
  \SetKwInOut{Input}{input}\SetKwInOut{Output}{output}
  \Input{Original plans $P = \{P_i\}_{i=0}^n$}
  \Output{Scheduled plans $P' = \{P'_i\}_{i=0}^n$}
  \BlankLine

  \For{$i=1$ \emph{to} $n$}
  {
    $P_i' \leftarrow P_i$ \\
    \For{$j=0$ \emph{to} $i-1$}
    {
      \While{$P_i'$ \emph{collide with} $P_j'$}
      {
        $P_i'$ postpone $\delta t$
      }
    }
  }

  return $P'$
  \caption{Inter-bridge scheduling}
  \label{algo:Planner}
\end{algorithm}

\vspace*{-0.15in}

\section{Completeness and Performance Analysis}
\label{sec:analysis}

Our bridge-based planning algorithm is probabilistically complete. In particular, our trajectory generation includes two parts: the trajectory generation inside the bridge and the trajectory generation outside the bridge. The former is probabilistically complete because the bridge is computed using RRT with the start and goal positions inside the agents' reachable sets, and thus we can always find a valid bridge with probability $1$ if one exists. Once the bridge is constructed, the trajectory generation inside the bridge is deterministic and thus is always possible.

The complexity of our method's online phase is $\mathcal O(n^2)$, where $n$ is the number of agents. In particular, the inter-trajectory scheduling in Section~\ref{sec:scheduling} dominates the running time cost. As shown in line $5$ of Algorithm~\ref{algo:Planner}, each agent needs to check whether its planned trajectory will collide with any plans of previously scheduled agents. If so, the agent will delay the plan to avoid collisions. For $n$ agents, we need to execute $\frac{1}{2}n^2$ checks in total. For each check between two trajectories, we need to further check whether any two segments from these paths collide with each other. For two long trajectories $P$ and $P'$, seems that we need to perform a pairwise checking with an expensive $|P|\times |P'|$ complexity. However, note that when agents are moving inside the same bridge, their speeds are determined by their position in the start line and are fixed while passing through the bridge (Theorem~\ref{thm:2d_bridge_interpolation}). As a result, agents falling behind will never catch up the agents in the front, and thus we can determine the collision status between two plans by only checking collisions between a few segments. In fact we observe that at most $C<10$ collision checks are performed between two trajectories in all our experiments. In this way, the computational complexity of the online phase is $C \times \frac{1}{2}n^2 = \mathcal O(n^2)$.

\section{Experiment}
\label{sec:experiment}

\begin{table}
\centering{}%
\begin{tabular}{|c|c|c|c|c|}
 \hline
 & \#Agent &\#Bridges & \# Overhead 1(ms)& \# Overhead 2(ms) \tabularnewline \hline
 Benchmark 1 &96& 2 & 204513 &1012\tabularnewline
 Benchmark 2 &96& 2 & 453267 &1432\tabularnewline
 Benchmark 3 &100& 2& 553462 &1431\tabularnewline
 Benchmark 4 &96& 2 & 1079381&1532\tabularnewline
\hline
\end{tabular}
\caption{The overhead 1 shows the time cost for constructing the given number of bridges in various benchmarks.
The overhead 2 shows the overhead of the scheduling algorithm.}
\label{tab:BuildCorridor}
\end{table}

\begin{figure*}[t]
\centering
\includegraphics[width=0.8\linewidth]{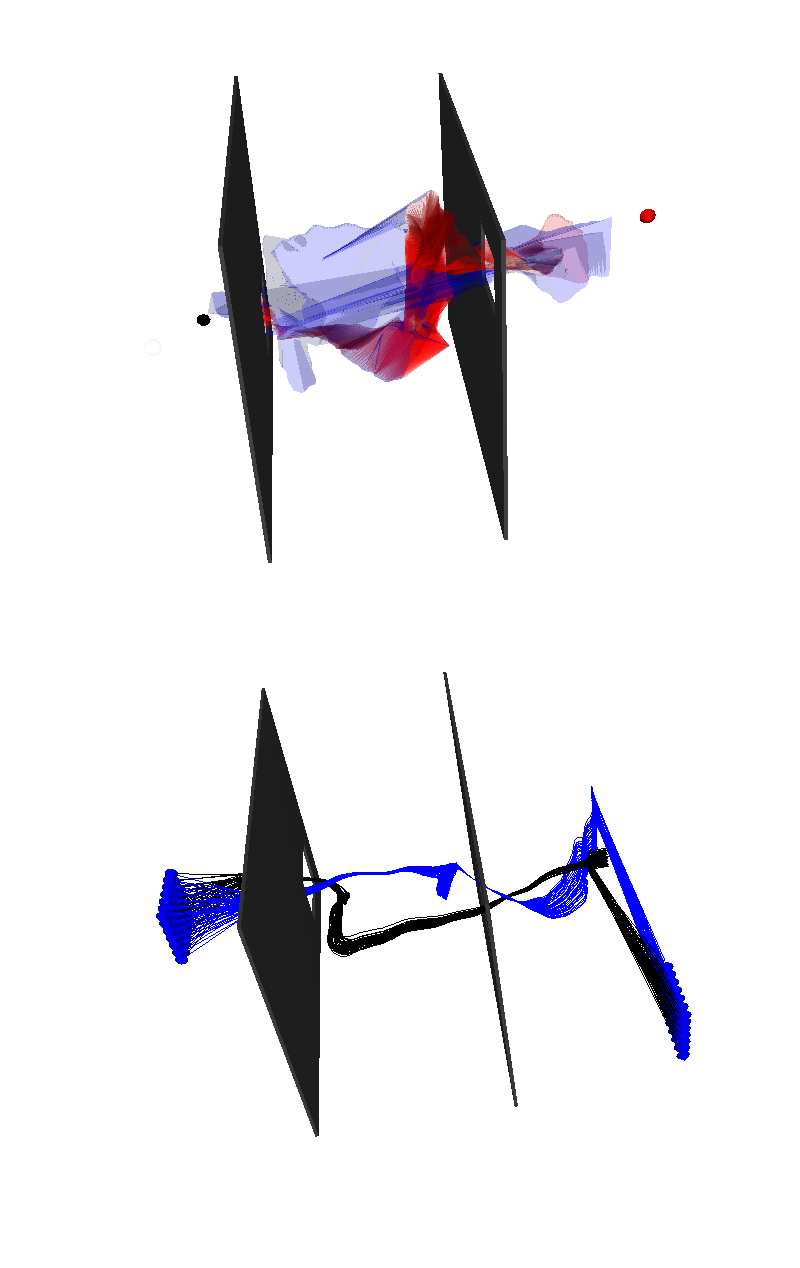}
\caption{This figure shows the bridges(above picture) and trajectories(below picture) of 3D scenario. In this benchmark, two groups of agents each has 49 agents exchange their positions.}
\label{fig:scenario1}
\end{figure*}

\begin{figure*}[t]
\includegraphics[width=0.8\linewidth]{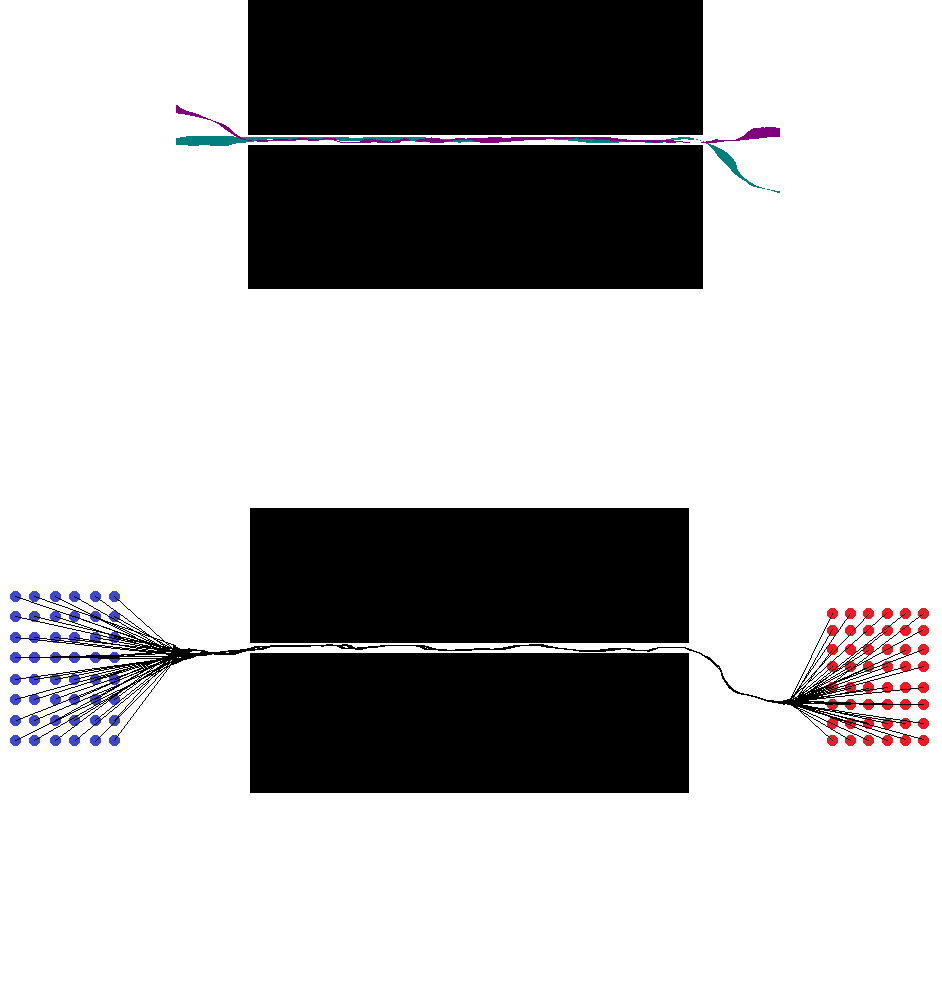}
\caption{This figure shows the bridges(above picture) and the trajectories(below picture) of 2D scenario 1. In this benchmark, two groups of agents each has 48 agents exchange their positions.}
\label{fig:scenario2}
\end{figure*}

\begin{figure*}[t]
\centering
\includegraphics[width=0.8\linewidth]{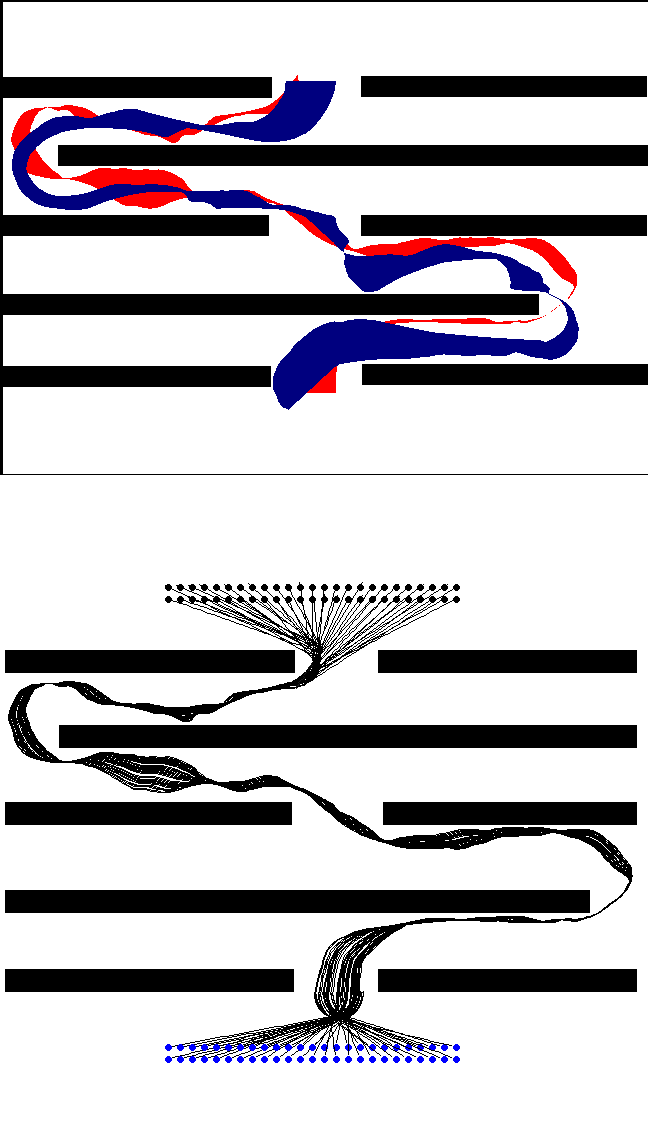}
\caption{This figure shows the bridges(above picture) and the trajectories(below picture) of 2D scenario 2. In this benchmark, two groups of agents each has 50 agents exchange their positions.}
\label{fig:scenario3}
\end{figure*}

\begin{figure*}[t]
\centering
\includegraphics[width=0.8\linewidth]{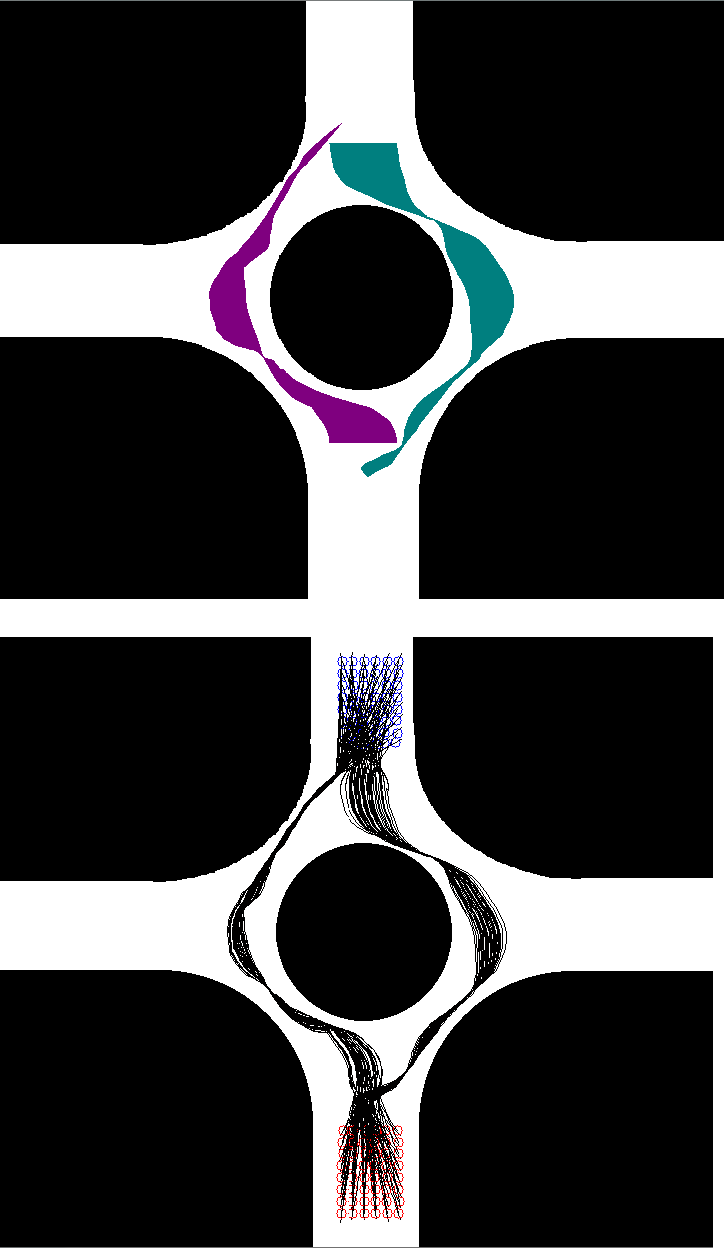}
\caption{This figure shows the bridges(above picture) and the trajectories(below picture) of 2D scenario 3. In this benchmark, two groups of agents each has 48 agents exchange their positions.}
\label{fig:scenario4}
\end{figure*}

\begin{table}
\centering{}%
\begin{tabular}{|c|c|c|}
 \hline
 & Outside time (ms) & Inside time (ms) \tabularnewline \hline
 Benchmark 1 & 23 & $<1$ \tabularnewline
 Benchmark 2 & 47 & $<1$ \tabularnewline
 Benchmark 3 & 132& $<1$ \tabularnewline
 Benchmark 4 & 560 & $<1$ \tabularnewline
\hline
\end{tabular}
\caption{The average time cost for each agent to compute its trajectory inside and outside the bridges on various benchmarks.}
\label{tab:TrajectoryBuilding1}
\end{table}

\begin{table}
\centering{}%
\begin{tabular}{|c|c|c|c|c|}
 \hline
  & Benchmark 1  & Benchmark 2  &Benchmark 3 & Benchmark 4 \tabularnewline \hline
 HRVO (\#collision times) &- &212 &- &517 \tabularnewline
HRVO+PRM (\#collision times) &13 &9 &11 &4 \tabularnewline
 OURS (\#collision times) &0 &0 & 0 &0 \tabularnewline
\hline
\end{tabular}
\caption{The comparison of collision times between our method, the local method(HRVO), and the local method plus traditional global method(PRM).}
\label{tab:comparison}
\end{table}

\begin{table}
\centering{}%
\begin{tabular}{|c|c|c|c|c|}
 \hline
  & Benchmark 1  & Benchmark 2  &Benchmark 3 & Benchmark 4 \tabularnewline \hline
 HRVO (\#frames) &- &3938 &- &4882 \tabularnewline
HRVO+PRM (\#frames) &3905 &2215 &1101 &752 \tabularnewline
 OURS (\#frames) &2024 &1482 & 762 &417 \tabularnewline
\hline
\end{tabular}
\caption{The comparison of simulation time between our method, the local method(HRVO), and the local method plus traditional global method(PRM). The '-' marks mean agents get stuck in the obstacles so they cannot achieve their goals.}
\label{tab:TrajectoryBuilding2}
\end{table}

In this section, we describe our implementation and highlight the improved performance of our navigation algorithm on four challenging benchmarks, of which two are 3D and the others are 2D as illustrated in Figure~\ref{fig:scenario1},~\ref{fig:scenario2},~\ref{fig:scenario3},~\ref{fig:scenario4}. We will show the pre-computation overhead and running performance. In the running stage, the overhead is divided into two parts, one for single-agent trajectory calculation and the other for multi-agent trajectory collision avoidance. The algorithm was introduced in the previous section. We have implemented our algorithms in C++ on an Intel Core i7 CPU running at 3.30GHz with 16GB of RAM on Window 7. All the timings are generated on a single CPU core

In all four benchmarks, we use the secondary linear dynamic system agents. $40-100$ agents are introduced for each scenario, evenly distributed around the start sides of the bridges. In the 3D scenario, the radius of each agent is 3 units and its maximum velocity is 2 units/sec. The maximum acceleration is $1$ $unit^2$/sec. The size of the scenario is $150 \times 80 \times 80$. In 2D the radius of the agent is $5$ units. The maximum velocity of an agent is $3$ units/sec. The maximum acceleration is $2$ $units^2$/sec. The scenario size is $750 \times 480$ units.

We highlight our performance in several substeps of our algorithm. The results are shown in Tables~\ref{tab:BuildCorridor}, ~\ref{tab:TrajectoryBuilding1}, ~\ref{tab:comparison}, and,~\ref{tab:TrajectoryBuilding2} for different benchmarks. We first analyze our online performance. In our result, the overhead for computing the global trajectory for each agents is very fast, especially when the agents are inside the bridge. In all the benchmarks, this substep of interpolating the trajectory in the bridge takes less than $1$ ms. When the agents are outside the bridge, the agents can still easily find the entrance of the bridge by using the optimal trajectory equation~\cite{webb2013kinodynamic}. This entrance-finding procedure introduces an overhead of less than $2$ seconds for the entire simulation.

Next, we highlight the performance of our pre-computation stage. In this case, the performance depends on the implementation of the RRT algorithm. In all the experiments, these precomputation overheads are in a reasonable range. Table ~\ref{tab:BuildCorridor} shows the performance of our algorithm. Here the left column shows the average length of agents’ trajectories, i.e. how many nodes of trajectories on average. Each right column shows the number of frames. Our algorithms generate very dense motion plans. All of the agents can finish their jobs in a reasonable time. In the worst case, for 40-100 agents, our algorithm can compute global collision-free paths in 30 seconds.
In order to highlight the advantages over local navigation (see Table~\ref{tab:BuildCorridor}), our method is compared with prior local navigation methods: RVO~\cite{Jur:2011:RVO} in 3D + PRM, HRVO~\cite{Snape:2014:SCN} in 2D + PRM. We describe the four scenarios and analyze the relative performance on these benchmarks.

\paragraph{Benchmark 1} This is a challenging scenario where 40 agents need to pass through two holes simultaneously and arrive at their respective goals, as shown in Figure~\ref{fig:scenario1}(a). The holes are small and only allow two agents to pass through at the same time. However, the areas around the agents' start and goal positions are widely open, and thus the agents can easily find a path to enter the bridges, as shown by the small outside bridge cost in Table~\ref{tab:BuildCorridor}. Compared to RVO algorithms, our method can effectively compute collision-free global paths, while the RVO local navigation methods get stuck in the narrow passages and can take a long time to make all agents reach the goals. In addition, our method results in fewer collision cases as compared to RVO and RVO+PRM, as shown in Table~\ref{tab:comparison}.

\paragraph{Benchmark 2:}
This is a 2D benchmark with a long narrow corridor where 96 agents are trying to move from one end to the other end, as shown in Figure~\ref{fig:scenario2}(b). The corridor only allows one agent to pass through at a time. For this scene, we compute two bridge connecting the two open spaces in the scenario. The RVO method can compute the collision-free trajectories for agents, but they tend to get stuck for a while due to the narrow corridor. The RVO+PRM method also results a large number of collisions. Our method generates a more stable and faster simulation than RVO and RVO+PRM, as shown in Figure~\ref{tab:comparison}.

\paragraph{Benchmark 3:}
This scenario has multiple narrow passages as shown in Figure~\ref{fig:scenario3}(c), and we construct two bridges to help the global navigation through these narrow passages. In this scenario, the RVO method fails to compute collision-free paths because agents are getting stuck in the narrow passages. The RVO+PRM method can find a feasible solution, but it takes significantly longer time as compared to our method and also results in more collisions, as demonstrated in Table~\ref{tab:comparison}.

\paragraph{Benchmark 4:}
In this benchmark, we have circle obstacle with narrow space in the scene.
Our method builds two bridges to connect the left and right areas in the workspace. The simulation result is shown in Figure~\ref{fig:scenario4}(d). As compared to local methods, our approach results in fewer agent-agent collisions and can compute the final trajectories faster as shown in Table~\ref{tab:comparison}.

\section{Conclusion and Future Work}
We present a novel multi-agent global navigation algorithm using interpolation bridges. Our approach is general and overcomes some of the major limitations of prior methods in terms of navigating through crowded areas or narrow regions. We present new techniques to compute these bridges in 2D and 3D workspaces and use their properties to compute interpolating collision-free trajectories for the agents. The construction of our bridge enables collision-free multi-agent global navigation. We have demonstrated its performance on many complex 2D and 3D scenarios and can perform collision-free navigation for tens of agents in realtime.

Our approach has some limitations. The bridge computation is limited to static obstacles, or dynamic obstacles whose trajectories are known. The complexity of global navigation increases with the number of bridges in the workspace, and very complex scenarios can result in a high number of bridges. Furthermore, our current approach is limited to agents with linear dynamics.
There are many avenues for future work. In addition to overcoming these limitations, we would like to design improved algorithms for bridge computation and further evaluate their performance.

\bibliographystyle{splncs}
\bibliography{references}

\begin{thebibliography}{10}

\bibitem{Fraichard:2003:ICS}
Fraichard, T., Asama, H.:
\newblock Inevitable collision states. a step towards safer robots?
\newblock In: IEEE/RSJ International Conference on Intelligent Robots and
  Systems. Volume~1. (2003)  388--393 vol.1

\bibitem{Fraichard:2007:SPM}
Fraichard, T.:
\newblock A short paper about motion safety.
\newblock In: IEEE International Conference on Robotics and Automation. (2007)
  1140--1145

\bibitem{webb2013kinodynamic}
Webb, D., van~den Berg, J.:
\newblock Kinodynamic rrt*: Asymptotically optimal motion planning for robots
  with linear dynamics.
\newblock In: IEEE International Conference on Robotics and Automation. (2013)
  5054--5061

\bibitem{Berenson:ICRA:RPP:2012}
Berenson, D., Abbeel, P., Goldberg, K.:
\newblock A robot path planning framework that learns from experience.
\newblock In: IEEE International Conference on Robotics and Automation. (2012)
  3671--3678

\bibitem{Jur:2011:RVO}
van~den Berg, J., Guy, S., Lin, M., Manocha, D.:
\newblock Reciprocal n-body collision avoidance.
\newblock In Pradalier, C., Siegwart, R., Hirzinger, G., eds.: Robotics
  Research. Volume~70 of Springer Tracts in Advanced Robotics.
\newblock Springer Berlin Heidelberg (2011)  3--19

\bibitem{Snape:2014:SCN}
Snape, J., Guy, S., van~den Berg, J., Manocha, D.:
\newblock Smooth coordination and navigation for multiple differential-drive
  robots.
\newblock In Khatib, O., Kumar, V., Sukhatme, G., eds.: Experimental Robotics.
  Volume~79 of Springer Tracts in Advanced Robotics.
\newblock Springer Berlin Heidelberg (2014)  601--613

\bibitem{John:1985:MPP}
Reif, J., Sharir, M.:
\newblock Motion planning in the presence of moving obstacles.
\newblock In: Symposium on Foundations of Computer Science. (1985)  144--154

\bibitem{Bekris:2012:SDM}
Bekris, K.E., Grady, D.K., Moll, M., Kavraki, L.E.:
\newblock Safe distributed motion coordination for second-order systems with
  different planning cycles.
\newblock The International Journal of Robotics Research \textbf{31}(2) (2012)
  129--150

\bibitem{he2013meso}
He, L., van~den Berg, J.:
\newblock Meso-scale planning for multi-agent navigation.
\newblock In: IEEE International Conference on Robotics and Automation. (2013)
  2839--2844

\bibitem{he2016proxemic}
He, L., Pan, J., Wang, W., Manocha, D.:
\newblock Proxemic group behaviors using reciprocal multi-agent navigation.
\newblock In: IEEE International Conference on Robotics and Automation. (2016)
  292--297

\bibitem{solovey2015motion}
Solovey, K., Yu, J., Zamir, O., Halperin, D.:
\newblock Motion planning for unlabeled discs with optimality guarantees.
\newblock CoRR \textbf{abs/1504.05218} (2015)

\bibitem{solovey2014hardness}
Solovey, K., Halperin, D.:
\newblock On the hardness of unlabeled multi-robot motion planning.
\newblock CoRR \textbf{abs/1408.2260} (2014)

\bibitem{Dresner:2008:MAA}
Dresner, K., Stone, P.:
\newblock A multiagent approach to autonomous intersection management.
\newblock Journal of Artificial Intelligence Resesearch \textbf{31}(1) (March
  2008)  591--656

\bibitem{Alami:1998:MRC}
Alami, R., Fleury, S., Herrb, M., Ingrand, F., Robert, F.:
\newblock Multi-robot cooperation in the martha project.
\newblock IEEE Robotics Automation Magazine \textbf{5}(1) (1998)  36--47

\bibitem{geraerts2008using}
Geraerts, R., Kamphuis, A., Karamouzas, I., Overmars, M.:
\newblock Using the corridor map method for path planning for a large number of
  characters.
\newblock In: Motion in Games.
\newblock Springer (2008)  11--22

\bibitem{Wein:2008:PHQ}
Wein, R., van~den Berg, J., Halperin, D.:
\newblock Planning high-quality paths and corridors amidst obstacles.
\newblock The International Journal of Robotics Research \textbf{27}(11-12)
  (2008)  1213--1231

\bibitem{Berg:2005:RMP}
van~den Berg, J., Overmars, M.:
\newblock Roadmap-based motion planning in dynamic environments.
\newblock IEEE Transactions on Robotics \textbf{21}(5) (2005)  885--897

\bibitem{Berg:2005:PMP}
van~den Berg, J., Overmars, M.:
\newblock Prioritized motion planning for multiple robots.
\newblock In: IEEE/RSJ International Conference on Intelligent Robots and
  Systems. (2005)  430--435

\bibitem{Lavalle:2001:RKP}
LaValle, S., Kuffner, J.:
\newblock Randomized kinodynamic planning.
\newblock International Journal on Robotics Research \textbf{20}(5) (2001)
  378--400

\bibitem{Karaman:2010:OKM}
Karaman, S., Frazzoli, E.:
\newblock Optimal kinodynamic motion planning using incremental sampling-based
  methods.
\newblock In: IEEE Conference on Decision and Control. (2010)  7681--7687

\bibitem{Cap:2013:MRS}
\v{C}\'{a}p, M., Nov\'{a}k, P., Vokr\'{\i}nek, J., P\v{e}chou\v{c}ek, M.:
\newblock Multi-agent rrt: Sampling-based cooperative pathfinding.
\newblock In: International Conference on Autonomous Agents and Multi-agent
  Systems. (2013)  1263--1264

\bibitem{Ferguson:2006:RRR}
Ferguson, D., Kalra, N., Stentz, A.:
\newblock Replanning with rrts.
\newblock In: IEEE International Conference on Robotics and Automation. (2006)
  1243--1248

\bibitem{Katsev:2013:EFP}
Katsev, M., Yu, J., LaValle, S.:
\newblock Efficient formation path planning on large graphs.
\newblock In: IEEE International Conference on Robotics and Automation. (2013)
  3606--3611

\bibitem{Sanchez:2002:UPP}
Sanchez, G., Latombe, J.C.:
\newblock Using a prm planner to compare centralized and decoupled planning for
  multi-robot systems.
\newblock In: IEEE International Conference on Robotics and Automation.
  Volume~2. (2002)  2112--2119 vol.2

\bibitem{Jur:2009:RSS}
van~den Berg, J., Snoeyink, J., Lin, M., Manocha, D.:
\newblock Centralized path planning for multiple robots: Optimal decoupling
  into sequential plans.
\newblock In: Robotics: Science and systems. (2009)

\bibitem{Luna:2011:PSF}
Luna, R., Bekris, K.E.:
\newblock Push and swap: Fast cooperative path-finding with completeness
  guarantees.
\newblock In: Joint Conference on Artificial Intelligence - Volume Volume One.
  (2011)  294--300

\bibitem{Gottschalk:1996:OHS}
Gottschalk, S., Lin, M.C., Manocha, D.:
\newblock Obbtree: A hierarchical structure for rapid interference detection.
\newblock In: SIGGRPAH. (1996)  171--180

\end{thebibliography}

\end{document}